\documentclass{article}

\usepackage{microtype}
\usepackage{graphicx}
\usepackage{booktabs} 
\usepackage{mathextra}
\usepackage{braket}
\usepackage{subcaption}
\usepackage{hyperref}
\usepackage{hyperref}

\usepackage[accepted]{icml2019_style/icml2019}

\icmltitlerunning{Adaptive and Safe Bayesian Optimization in High Dimensions via One-Dimensional Subspaces}

\begin{document}

\twocolumn[
\icmltitle{Adaptive and Safe Bayesian Optimization in High Dimensions via One-Dimensional Subspaces}



\icmlsetsymbol{equal}{*}

\begin{icmlauthorlist}
\icmlauthor{Johannes Kirschner}{eth}
\icmlauthor{Mojm\'ir Mutn\'y}{eth}
\icmlauthor{Nicole Hiller}{psi}
\icmlauthor{Rasmus Ischebeck}{psi}
\icmlauthor{Andreas Krause}{eth}
\end{icmlauthorlist}

\icmlaffiliation{eth}{
	Department of Computer Science,
	ETH Zurich, Switzerland\linebreak
}
\icmlaffiliation{psi}{
	Paul Scherrer Institut,
	Switzerland}
\icmlcorrespondingauthor{\linebreak Johannes Kirschner}{jkirschner@inf.ethz.ch}

\icmlkeywords{Bayesian Optimization, black box optimization, Gaussian processes}

\vskip 0.3in
]



\printAffiliationsAndNotice{}  


\begin{abstract}
Bayesian optimization is known to be difficult to scale to high dimensions, because the acquisition step requires solving a non-convex optimization problem in the same search space. In order to scale the method and keep its benefits, we propose an algorithm (\textsc{LineBO}) that restricts the problem to a sequence of iteratively chosen one-dimensional sub-problems that can be solved efficiently. We show that our algorithm converges globally and obtains a fast local rate when the function is strongly convex. Further, if the objective has an invariant subspace, our method automatically adapts to the effective dimension without changing the algorithm. 
When combined with the \textsc{SafeOpt} algorithm to solve the sub-problems, we obtain the first safe Bayesian optimization algorithm with theoretical guarantees applicable in high-dimensional settings. We evaluate our method on multiple synthetic benchmarks, where we obtain competitive performance. Further, we deploy our algorithm to optimize the beam intensity of the Swiss Free Electron Laser with up to 40 parameters while satisfying safe operation constraints.
\end{abstract}

\section{Introduction}

	Zero-order stochastic optimization problems arise in many applications such as hyper-parameter tuning of machine learning models, reinforcement-learning and industrial processes. An example that motivates the present work is parameter tuning of a free electron laser (FEL). FELs are large-scale physical machines that accelerate electrons in order to generate bright and shortly pulsed X-ray lasing. The X-ray pulses then facilitate many experiments in biology, medicine and material science. The accelerator and the electron beam line of a free electron laser consist of multiple individual components, each of which has several parameters that experts adjust to maximize the pulse energy. Because of different operational modes and parameter drift, this is a recurrent, time-consuming task which takes away valuable time for experiments. As a single measurement can be obtained in less than one second, the task is well suited for automated optimization with a continuous search space of about 10-100 parameters. Further, some parameters are known to physically over-parametrize the objective function, which leads to invariant subspaces and also local optima. Additionally, some settings can cause electron losses, which are required to stay below a pre-defined threshold.

	\begin{figure}[t]
			\begin{subfigure}[t]{0.15\textwidth} \vskip 0pt
			\includegraphics[height=90px, trim={0cm 0 40px 0},clip]{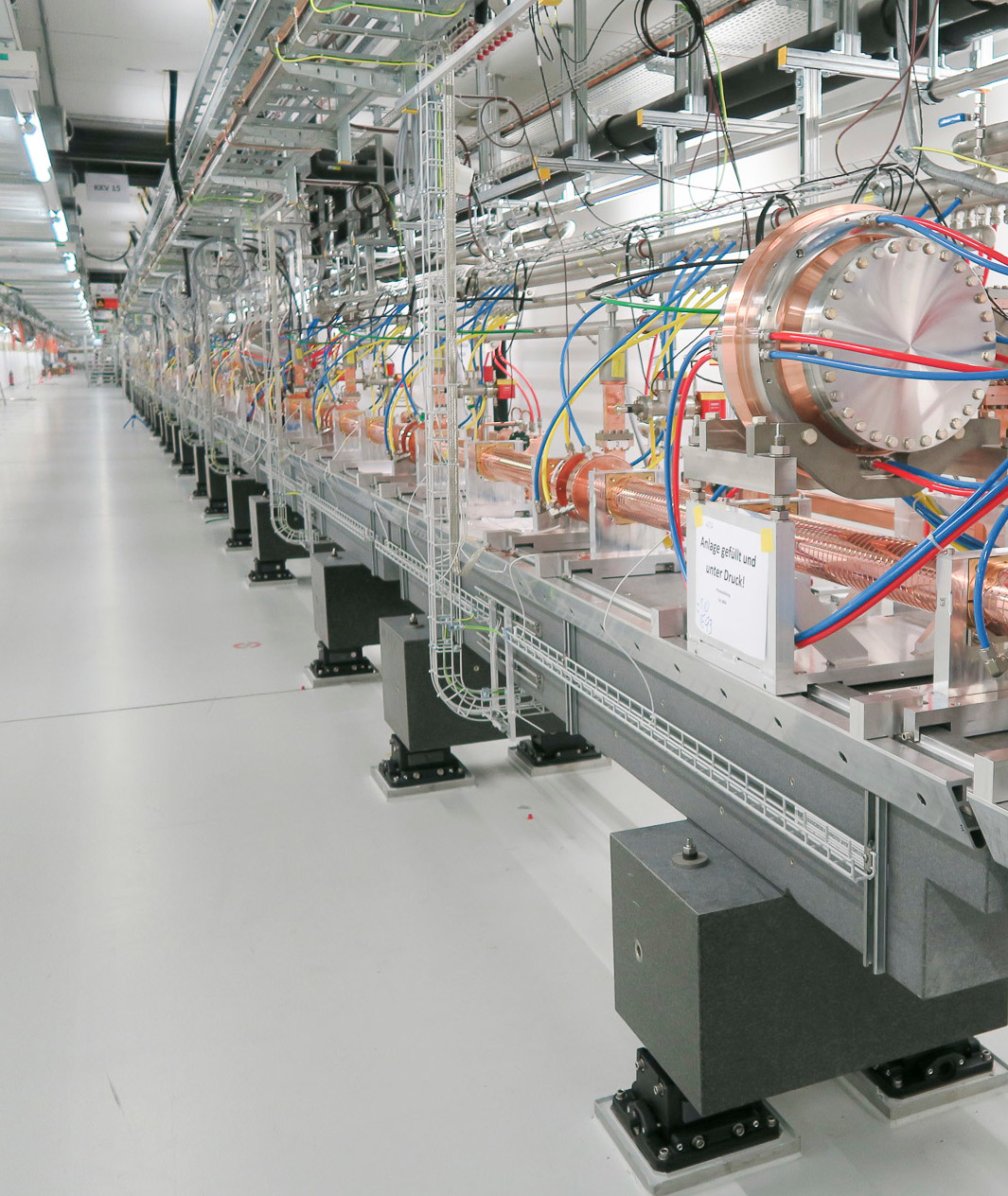}
		\end{subfigure}\hfill
		\begin{subfigure}[t]{0.3\textwidth}\vskip 0pt
			\includegraphics[trim={0 2px 0 1px},clip]{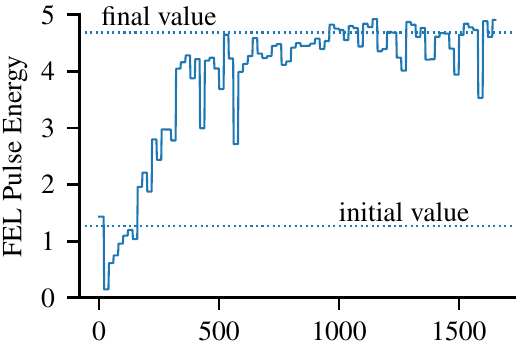}
		\end{subfigure}\hspace{5px}
		\caption{Left: Inside a free electron laser tunnel. Right: Using \textsc{LineBO} to tune the SwissFEL pulse energy with 40 parameters.}
		\label{fig:swissfel40d}
	\end{figure}
	
	This scenario can be cast as a gradient-free stochastic optimization problem with implicit constraints. The fact that the constraints are safety critical rules out many commonly used algorithms. Arguably, the simplest approach is to use a local optimization method with a conservatively chosen step size and a term that penalizes constraint violations in the objective, but such a method might get stuck in local optima. As an alternative, Bayesian optimization offers a principled, global optimization routine that can also operate under safety constraints \citep{Sui2015}. When applied to a low-dimensional subset of parameters, Bayesian optimization has been successfully used on FELs and in similar applications. However, it is well known that standard Bayesian optimization is difficult to scale to high-dimensional settings, because optimizing the acquisition function becomes itself an intractable optimization problem.
	
	In this work, we propose a novel way of using Bayesian optimization that is computationally feasible \emph{even in high dimensions}. The key idea is to iteratively solve sub-problems of the global problem, each of which can be solved efficiently, both computationally and statistically. As feasible sub-problems we choose one-dimensional subspaces of the domain that contain the best point so far. On a one-dimensional domain, Bayesian optimization can be implemented computationally efficiently and the sample-complexity to obtain an $\epsilon$-optimal point is independent of the outer dimension. A global GP model can nevertheless be used and allows to \emph{share information} between the sub-problem to increase data-efficiency, in particular as samples start to accumulate close to an optimum. As we will show, our approach obtains \emph{both local and global convergence guarantees} and further adaptively scales with the \emph{effective dimension}, if the objective contains an invariant subspace. In the constraint setting, we use \textsc{SafeOpt} to solve the sub-problems. This way, we obtain the first principled and safe Bayesian optimization algorithm applicable to high-dimensional domains.

	\subsection{Contributions}
\begin{itemize}	
		\item We propose a novel way of using Bayesian optimization that circumvents the issue of acquisition function optimization by decomposing the global problem into a sequence of one-dimensional sub-problems that can be solved efficiently. 
		\item Theoretically, we show that if the one-dimensional subspaces are chosen randomly, the algorithm converges with \emph{a fast local rate} where the function is strongly convex, and converges \emph{globally at a Lipschitz rate} that adaptively scales with the effective dimension.
		\item To respect safety constraints during optimization, each sub-problem can be solved with \textsc{SafeOpt}. To the best of our knowledge, this is the \emph{first principled algorithm for high dimensional safe Bayesian optimization}.
		\item Our algorithm is practical and amenable to heuristics that improve local convergence. As user feedback we provide one-dimensional slice plots that allow to monitor the progress and the model fit.
		\item We evaluate our method on synthetic benchmark functions, and apply it to tune the Swiss Free Electron Laser (SwissFEL) with up to 40 parameters on a continuous domain, satisfying safe operation constraints.
	\end{itemize}
				
	\subsection{Related Work}
		Derivative-free stochastic optimization covers an array of algorithms from the very general grid-based methods \cite{Nesterov2010, Jones2001} to local methods, where most of the work is spent on approximating the gradient \cite{Nesterov2017}. Especially of interest are algorithms that optimize functions with a noisy oracle, also known as stochastic bandit feedback \cite{Flaxman2004,Shamir2013}. Popular examples include CMA-ES \cite{Hansen2001, Hansen2003}, Nelder-Mead \cite{Powell1973} and SPSA \cite{Bhatnagar2013}. Line-search techniques are related to our method, but have been primarily studied in the context of convex optimization \cite{Gratton2015}, also with stochastic models and search directions \cite{Cartis2018, Paquette2018, Diniz-Ehrhardt2008}.

Bayesian optimization is a family of algorithms using probabilistic models to determine which point to evaluate next \cite{Mockus1982, Shahriari2016}. Many variants appeared in literature; including GP-UCB \cite{Srinivas2009}, Thompson Sampling \cite{Chowdhury2017}, and Expected Improvement \cite{Mockus1982}; and recently with information theoretic criteria such as MVES or IDS \cite{Wang2017,Kirschner2018}. Lower bounds are known as well \cite{Scarlett2017}. Bayesian optimization on a one dimensional domain is not necessarily thought of as line search, although it can be used as such \cite{Mahsereci2015}, and the one dimensional setting is theoretically well understood \cite{Scarlett2018}. Success stories, where Bayesian optimization outperforms classical techniques, include applications in laser technology \cite{Schneider2018}, performance optimization of Free Electron Lasers \cite{McIntire2016, McIntire2016b} and parameter optimization in the CPLEX suite \cite{Shahriari2016}.

The scaling of Bayesian optimization to high dimensions has been considered recently, as many of the commonly used kernels suffer from the curse of dimensionality. Hence, to make the problem tractable, most approaches make structural assumptions on the function such as additivity \cite{Rolland2018, Mutny2018} or a low-dimensional active subspace \cite{Djolonga2013}. The latter category includes REMBO \cite{Wang2016}, which also optimizes on a random low-dimensional subspace, however, in contrast to our method, the dimension of the low-dimensional embedding needs to be known a priori. An iterative procedure to define the subspaces is proposed by \citet{qian2016derivative} and similarly our method relates to the Dropout-BO algorithm of \citet{Li2017}, but in both cases the convergence analysis is incomplete.  A heuristic that combines local optimization with Bayesian optimization was proposed by \citet{McLeod2018}.

The main instance of safe Bayesian optimization is the \textsc{SafeOpt} algorithm \cite{Sui2015,Berkenkamp16, Sui2018}; but its formulation relies on a discretized domain, which prevents high-dimensional applications. An adaptive discretization based on particle swarms was proposed by \citet{Berkenkamp2016A}.

	\section{Problem statement}
		Let $\xX \subset \RR^d$ be a compact domain and $f : \xX \rightarrow \RR$ the objective function we seek to minimize\footnote{Bayesian optimization is typically formulated as maximization problem, but since we also have results in the flavor of convex optimization,  w.l.o.g.\ we use minimization here.}, 
		\begin{align}
			\min_{x\in \xX} f(x) \quad \text{ s.t. } \quad g(x) \leq 0 \text{,}
		\end{align}
		where we allow for implicit constraints $g:\xX \rightarrow \RR$. The constraint function can be chosen vector valued in the case of multiple constraints. We refer to such constraints as \emph{safety constraints} if it is required that the iterates $x_t$ satisfy $g(x_t) \leq 0$ during optimization. We assume that $f$ and $g$ can only be accessed via a noisy oracle, that given a point $x \in \xX$ returns an evaluation $y = f(x) + \epsilon$ and $s = g(x) + \epsilon'$, where $\epsilon$ is a noise term with sub-Gaussian tails.
		
		Denote $f^* = \min_{x\in \xX} f(x)$ and let $x^* \in \xX$ be a point such that $f(x^*) = f^*$. An optimization algorithm iteratively picks a sequence of evaluations $x_1,\dots, x_T$, and obtains the corresponding noisy observations $y_1,\dots, y_T$. As a measure of progress we use \emph{simple regret}. At any stopping time $T$, the optimization algorithm proposes a candidate solution $\hat{x}_T$. This point is allowed to differ from the point $x_T$ that is chosen for the purpose of optimization, still, some algorithms might set $\hat{x}_T = x_T$. Simple regret is defined as
		\begin{align}
			r_T := f(\hat{x}_T) - f^* \text{,}
		\end{align}
		and therefore measures the ability of an optimization algorithm to predict a minimizer at time $T$. To impose some regularity on $f$, we make the following assumption.
		
		\begin{assumption}[RKHS]\label{ass:f_regular}
			 The objective and constraint functions $f$ and $g$ are members of reproducing kernel Hilbert spaces $\hH(k_1)$, $\hH(k_2)$ with known kernel functions $k_1,k_2:\xX \times \xX \rightarrow \RR$ and bounded norm $\|f\|_{\hH_1},\|g\|_{\hH_2} \leq B$.
		 \end{assumption}
		
		 This assumption is central for Bayesian optimization, as it justifies the use of Gaussian processes to estimate $f$ from the samples \cite{Rassmussen2004,Kanagawa2018}.

\section{Line Bayesian Optimization}	
\begin{algorithm}[t]
	\caption{Line Bayesian Optimization (\textsc{LineBO})}\label{alg:linebo}
	\begin{algorithmic}[1]
		\REQUIRE Direction oracle $\Pi$, accuracy $\epsilon$, starting point $\hat{x}_0$, \\
		Model $\mM_0 = (\text{GP prior for } f,g)$
		\FOR{$i=1,2, \dots, K$}
		\STATE $l_i \gets \Pi(\mM_{i-1})$ \hfill \textit{// define direction}
		\STATE $\lL_i \gets \lL(\hat{x}_{i-1}, l_i)$ \hfill \textit{// define subspace}
		\STATE $\hat{x}_{i}, \mM_{i} \gets \text{BayesianOptimization}(\mM_{i-1}, \lL_{i}, \epsilon)$ \\ \hfill \textit{// includes posterior updates (Appendix \ref{app:BO})}
		\ENDFOR
	\end{algorithmic} 
\end{algorithm}

In its standard formulation, Bayesian optimization uses a Gaussian process prior GP$(\mu, k)$ with mean $\mu : \xX \rightarrow \RR$ and kernel function $k : \xX \times \xX \rightarrow \RR$ and Bayes' rule to update the posterior as observations $(x_t,y_t)$ arrive. If a Gaussian likelihood $\epsilon \sim \nN(0, \sigma^2)$ is used, the posterior mean $\hat{f}_t$ can be computed analytically and is equivalent to the regularized least squares kernel estimator,
\begin{align*}
\hat{f}_t(x) := \argmin_{f \in \hH_k} \sum_{t=1}^T \big(f(x_t) - y_t\big)^2 + \|f\|_{\hH_k}^2 \text{ .}
\end{align*}

From the Bayesian posterior, one can obtain credible intervals $\hat{f}_t(x) \pm \beta_t \sigma_t(x)$, which in this case are known to match frequentist confidence intervals up to the scaling factor $\beta_t$. Bayesian optimization is built upon using the uncertainty estimates $\sigma_t$, or more generally the posterior distribution, to determine promising query points $x_t$ that efficiently reduce the uncertainty about the true maximizer $x^*$. Typically, an acquisition function $\alpha_t(x) := \alpha(x|\hat{f}_t, \sigma_t)  : \xX \rightarrow \RR$ is defined to trade-off between exploration and exploitation on the GP posterior landscape and evaluations are chosen as $x_t \in \arg\max_{x \in \xX} \alpha_t(x)$. Commonly used acquisition functions include UCB, Thompson Sampling, Expected Improvement and Max-Value Entropy Search.

The success of Bayesian optimization crucially relies on the ability to find a maximizer of the acquisition function $\alpha_t$, which requires solving a non-convex optimization problem in the same search space $\xX$. In most of the literature on Bayesian optimization, this is not discussed further as the computational cost of solving $\arg\max \alpha_t(x)$ is assumed to be negligible compared to obtaining a new evaluation on the oracle. \emph{In practice, however, this step renders the method intractable in high-dimensional settings.} 

In order to maintain tractability of the acquisition step in high dimensions, we propose to restrict the search space to a one-dimensional\footnote{Generalization to higher dimensional subspaces is possible.} affine subspace $\lL(x,l) := \{ x + \alpha l : \alpha \in \RR \}\cap \xX$, where $x \in \xX$ is the offset, and $l \in \RR^d$ is the direction. On such a restriction, the acquisition step can be effectively solved using an (adaptive) grid-search over $\lL$. We will show that by carefully choosing a sequence $\lL_1,\dots, \lL_K$ of one-dimensional subspaces, we obtain a method that still converges globally and additionally has properties similar to a gradient method. By using a global GP model, we can share information between the sub-solvers and handle noise in a principled way.

The \textsc{LineBO} method is presented in Algorithm \ref{alg:linebo}. As standard for Bayesian optimization, we initialize with a GP prior. We also assume that the user provides a \emph{direction oracle} $\Pi$, which is used to iteratively define subspaces $\lL_i=\lL(x_i, l_i)$. The affine subspace is always chosen to contain the previous best point to ensure a monotonic improvement over $K$ iterations. We then proceed by efficiently solving the subspace $\lL_i$ using standard Bayesian optimization (Appendix \ref{app:BO}).

A canonical example of the direction oracle is to pick the direction uniformly at random, which is also the main focus of our analysis. As we will see, this algorithm obtains both a local and a global convergence rate. Another possibility is to use (random) coordinate aligned directions, which resembles a coordinate descent algorithm. In this case, our method is a special case of DropoutUCB of \citet{Li2017}, but the global rate they obtained has a non-vanishing gap in the limit and local convergence was not analysed.

\subsection{Safe Line Bayesian Optimization}

The restriction of the search space allows us to effectively use a safe Bayesian optimization algorithm like \textsc{SafeOpt} (see Appendix \ref{app:safeopt}) as a sub-solver, which in turn renders the global method safe (\textsc{SafeLineBO}). We note that in its current formulation, \textsc{SafeOpt} crucially relies on a discretized domain, which makes it difficult to apply even with $d > 3$; but it is an easy task to implement the method on a one dimensional domain. To the best of our knowledge, this way we obtain the first principled method for safe Bayesian optimization in high dimensions.

\section{Convergence Analysis}

\subsection{Sample Complexity of 1D Bayesian Optimization}

To understand the sample complexity of solving the one dimensional sub-problems, we rely on the standard analysis of Bayesian optimization developed by \citet{Srinivas2009,AbbasiYadkori2012,Chowdhury2017}. The results are often stated in terms of a complexity measure called \emph{maximum information gain} $\gamma_T$, which is defined as the mutual information $\gamma_T := \max_{A \subset \xX: |A|=T} I(y_A, f_A)$. This quantity depends on the kernel and upper bounds are known for the RBF and Matern kernel \cite{Seeger2008, Srinivas2009}. We focus on a subset of kernels, which when restricted on the one dimensional affine subspace $\lL$, their $\gamma_T(k|_\lL)$ satisfies the following assumption. 

\begin{assumption}[Bounded $\gamma_T$] \label{ass:bounded_gamma}
	Let $k:\RR\times \RR \rightarrow \RR^+ $ be a one-dimensional kernel and $\kappa \in (0,0.5)$, then 
	\begin{align*}
	\gamma_T(k) \leq \oO ( T^\kappa \log T).
	\end{align*}	
\end{assumption}

This is satisfied for the squared exponential kernel ($\kappa = 0$) and the Matern kernel with $\nu > \frac{3}{2}$ ($\kappa = \frac{2}{2v+2}$). Simple regret can be bounded as $r_T \leq \oO(\gamma_T / \sqrt{T})$, and with the assumption above, the bound becomes
$r_T \leq \oO(T^{\kappa - 1/2})$ up to logarithmic factors (see also Appendix \ref{app:sample_complexity_bo}). Equivalently, the time until $\epsilon$ regret is guaranteed is $T \leq \oO(\epsilon^{- \frac{2}{1-2\kappa}})$. The best known lower bound for this case is $r_T \geq \Omega(\epsilon^{-\frac{2}{1-\kappa}})$ \cite{Scarlett2017}, hence almost closes the gap. The overall number of evaluations after $K$ iterations of Algorithm \ref{alg:linebo} is at most $\oO(K \epsilon^{- \frac{2}{1-2\kappa}})$.

\subsection{Global Convergence and Subspace Adaptation}

In practice, we often encounter functions that are high-dimensional but contain an (unknown) invariant subspace. This means that there are directions in which the function is constant and after removal of these dimensions the problem might not be high dimensional (see Figure \ref{fig:subspace}). The dimension of the linear space where the function varies is called \emph{effective dimension}, as formalized in the following definition.

\begin{definition}[Effective dimension]\label{def:effective_dim} The effective dimensionality of a function $f:\RR^d \rightarrow \RR$ is the smallest $d_e \leq d $ s.t. there exists a linear subspace $\yY \subset \RR^d$ of dimension $d_e$ and for all   $x_\top \in \yY$ and $x_\perp \in \yY_\perp$, where $\yY_\perp$ is the orthogonal complement of $\yY$, $f(x_\top \oplus x_\perp) = f(x_\top \oplus 0)$.
\end{definition} 

\begin{figure}
\centering
\begin{subfigure}[b]{0.22\textwidth}
	\includegraphics[height=80px,trim={0cm 0 0px 0},clip]{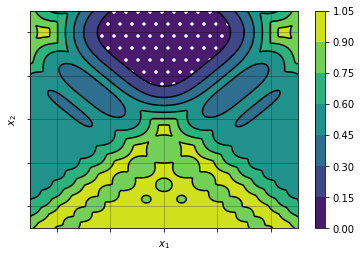}
\end{subfigure}\hspace{5px}
\begin{subfigure}[b]{0.22\textwidth}
	\includegraphics[height=80px]{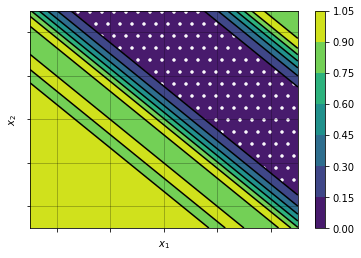}
\end{subfigure}
\caption{Function with $d_e = 2$ and $d_e = 1$. The volume of the set $V_\epsilon = \{ x | f(x) - f(x^*) \leq \epsilon \}$  (dotted region) for $d_e = 2$ and $d_e = 1$ can be significantly larger if the function contains an invariant subspace, which facilitates random exploration.}
\label{fig:subspace}
\end{figure}

If Algorithm \ref{alg:linebo} is used with randomly chosen directions, we show that the convergence of the algorithm adaptively scales with the effective dimension $d_e$. The result is quantified in the following proposition.

\begin{proposition}[Global convergence]\label{prop:global}
Let $f$ satisfy Assumption \ref{ass:f_regular} with effective dimension $d_e$, $k$ be twice differentiable, and let $\delta \in (0,1)$. Then after $K$ iterations of Algorithm \ref{alg:linebo} with accuracy $\epsilon$ and directions chosen uniformly at random, with probability at least $1-\delta$, it holds that
\begin{align*}
f(\hat{x}_K) - f^* \leq \oO\left( \left( \frac{1}{K} \log\left(\frac{1}{\delta}\right)  \right)^{\frac{2}{d_e -1 }} + \epsilon\right) \text{ .}
\end{align*}
\end{proposition}

The proof is deferred to Appendix \ref{app:global_convergence}. The result should be understood as a property of random exploration and is the best one can hope for on worst-case examples. Instances, where random search is competitive have been reported in literature \cite{Bergstra2012, Wang2016, Li2017Hyperband} and this has been attributed to the same effect. However, random search fails to control the error induced by the noise, and our method has the advantage of using the GP model to deal with the noise in a principled way.

In contrast to other algorithms that exploit subspace structure, including the REMBO algorithm of \citet{Wang2016} and SI-BO of \citet{Djolonga2013}, our formulation does \emph{not require} the knowledge of $d_e$ in advance. Intuitively, we can demonstrate the consequence of the effective dimension and the random line algorithm by plotting the set $V_\epsilon := \{x | f(x) - f^* \leq \epsilon\}$ that appears as an isolated spike in the domain in the worst-case. For functions with an invariant subspace, the volume of the set  $V_\epsilon$ increases substantially and hence the probability of a random line passing through this region increases (see Figure \ref{fig:subspace}). 

Naturally, such a bound cannot avoid an exponential scaling with $d_e$, as also does not full-scale Bayesian optimization even when restricted to the effective subspace. However, we show in the next section, that if our algorithm finds a point in the proximity of a local optimum, the convergence is dominated by a \emph{fast} local rate, a property not exhibited by random search.

\begin{figure*}[t]
	\centering
	\begin{subfigure}[b]{\textwidth}
		\vspace{-5px}
	\hspace{-6px}	\includegraphics{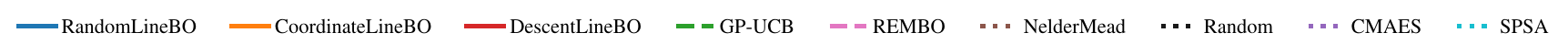}
	\end{subfigure}
	\begin{subfigure}[t]{1em}
		\vskip -92pt
		\rotatebox{90}{\small{simple regret}}
	\end{subfigure}
	\begin{subfigure}[b]{0.32\textwidth}
		\includegraphics{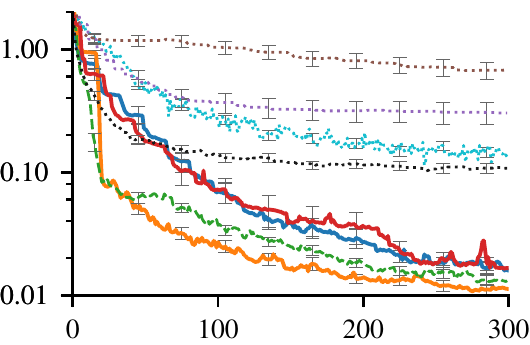}
		\caption{Camelback2D}
		\label{fig:camelback}
	\end{subfigure}\hfill%
	\begin{subfigure}[b]{0.32\textwidth}
		\includegraphics{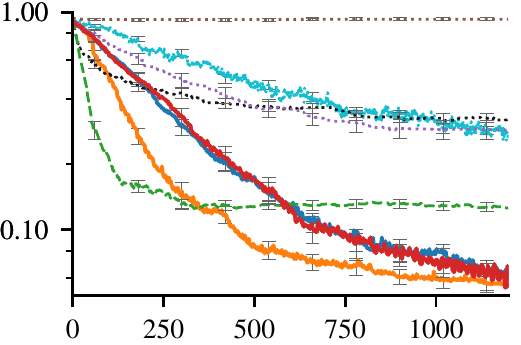}
		\caption{Hartmann6D}
		\label{fig:hartmann6}
	\end{subfigure}\hfill%
	\begin{subfigure}[b]{0.32\textwidth}
		\includegraphics{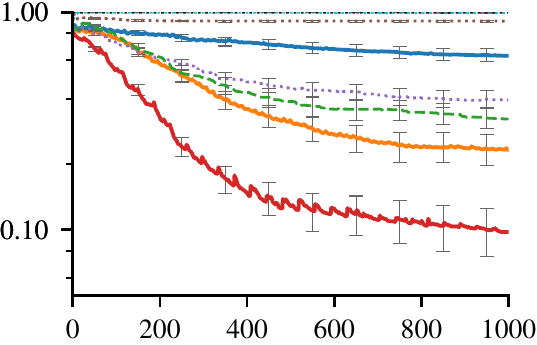}
		\caption{Gaussian10D}
		\label{fig:gaussian10}
	\end{subfigure}
	\begin{subfigure}[t]{1em}
	\vskip -92pt
	\rotatebox{90}{\small{simple regret}}
\end{subfigure}
	\begin{subfigure}[b]{0.32\textwidth}
		\includegraphics{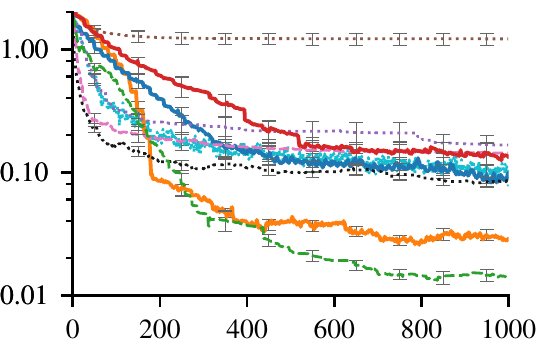}
		\caption{Camelback2D+10D}
		\label{fig:camelback_sub}
	\end{subfigure}\hfill%
	\begin{subfigure}[b]{0.32\textwidth}
		\includegraphics{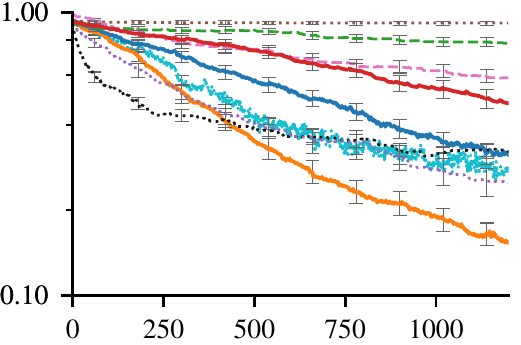}
		\caption{Hartmann6D+14D}
		\label{fig:hartmann6_sub}
	\end{subfigure}\hfill%
	\begin{subfigure}[b]{0.32\textwidth}
		\includegraphics{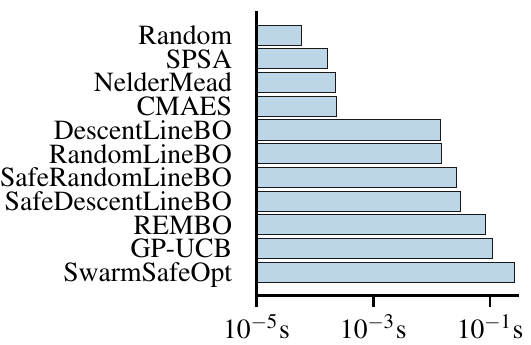}
		\caption{Average time per iteration}
		\label{fig:evaltime}
	\end{subfigure}
	\caption{We compare on standard functions, Camelback (\subref{fig:camelback}) and Hartmann6 (\subref{fig:hartmann6}). A 10d Gaussian (\subref{fig:gaussian10}) is used to demonstrate local convergence, with a starting point such that picking up the gradient signal is difficult. We further add invariant subspaces (\subref{fig:camelback_sub}, \subref{fig:hartmann6_sub}). Figure (\subref{fig:evaltime}) shows per-iteration computation time on a 10d benchmark. Naturally, the model-free approaches are quite fast, whereas the Bayesian optimization method have the computational burden of the GP model. However, restricting the possible acquisition space improves the per-step computation time by one order of magnitude in our implementation compared to the standard GP-UCB in our implementation. }
	\label{fig:results_unconstraint}
\end{figure*}

\begin{figure*}[t]
	\centering
	\begin{subfigure}[t]{\textwidth}
		\vspace{-8px}
	\hspace{40px}	\includegraphics{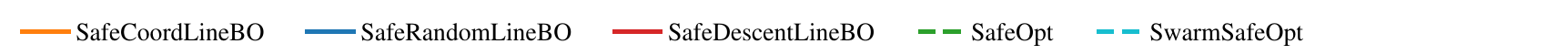}
	\end{subfigure}
	\begin{subfigure}[t]{1em}
	\vskip -80pt
	\rotatebox{90}{\small{simple regret}}
	\end{subfigure}
	\begin{subfigure}[t]{0.32\textwidth}
		\includegraphics{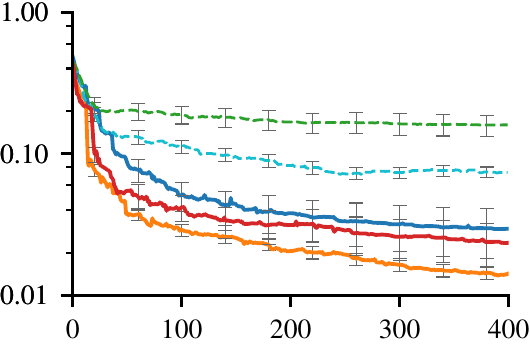}
		\caption{Camelback2D-Constraint}
		\label{fig:camelback_constraint}
	\end{subfigure}
	\begin{subfigure}[t]{0.32\textwidth}
		\includegraphics{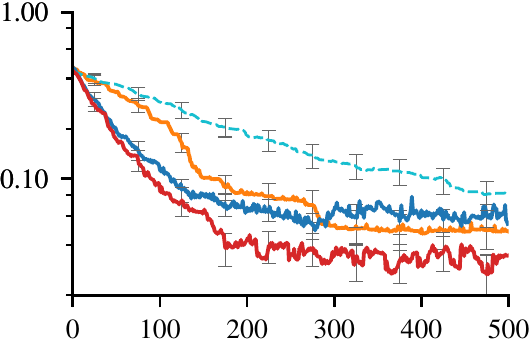}
		\caption{Camelback2D+10D-Constraint}
		\label{fig:camelback_sub_constraint}
	\end{subfigure}
	\begin{subfigure}[t]{0.32\textwidth}
		\includegraphics{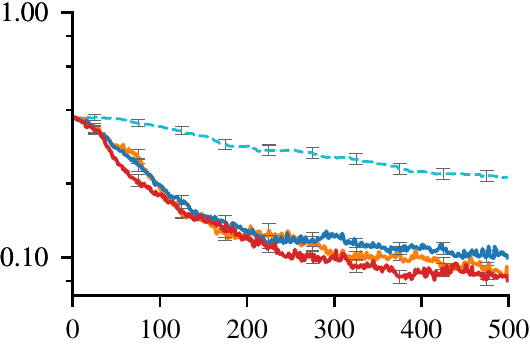}
		\caption{Hartmann6D-Constraint}
		\label{fig:hartmann6_constraint}
	\end{subfigure}
	\caption{We compare on standard benchmarks with additional constraints. Note that \textsc{SafeOpt} relies on a discretized domain and is therefore not applicable in the high-dimensional benchmarks. We found that performance of the methods strongly depends on the initial point, here we show average performance over a starting point chosen uniformly random in the safe set.}
	\label{fig:constraint}
\end{figure*}

\subsection{Local Convergence}
By Taylor's theorem, differentiable functions have an open set around their minimizers where the function is convex or even strongly-convex. We show that if our algorithm starts in a subset of the domain where the function is strongly convex, it converges to the (local) minimum at a linear rate. Again, we focus on the instance where directions are picked at random. The key insight is that random directions can be used as descent directions in the following sense.

\begin{lemma}[Random Descent Direction] \label{lem:random_directions}
	Let $l \in \RR^d$ be a uniformly random point on the $d$-dimensional unit sphere or uniformly among an orthonormal basis. Then,
	\begin{align*}
\text{for all } x \in \xX, \qquad \EE[\< \nabla f(x), l \>^2] = \frac{1}{d} \|\nabla f(x)\|^2   \text{ .}
	\end{align*}
\end{lemma}
The standard proof technique for descent algorithms on strongly convex functions \cite{Nesterov2012} yields the following result; see Appendix \ref{app:local_convergence} for a proof.

\begin{proposition}\label{prop:local_convg}
	Let $f$ satisfy Assumption \ref{ass:f_regular}, be $\alpha$-strongly convex and $\beta$-smooth if restricted to $\xX_c \subset \xX$. Let $f_c^* = \max_{x \in \xX_c} f(x)$ and assume all iterates  $\hat{x}_k$ are contained in $\xX_c$. Then, after $K$ iterations of Algorithm \ref{alg:linebo} with accuracy $\epsilon$ and random directions that satisfy Lemma \ref{lem:random_directions}, it holds that,
	\begin{align*}
	\EE[f(\hat{x}_K)] - f^*_c \leq \frac{\epsilon \beta d}{\alpha} + \left(1-\frac{\alpha}{\beta d}\right)^K(f(x_0) - f^*_c).
	\end{align*}
\end{proposition}
	To interpret the result, we fix the total number of evaluations $T$ and assume $f^*_c = f^*$. If the kernel $k$ restricted to any one dimensional subspace satisfies Assumption \ref{ass:bounded_gamma}, we can set the accuracy $\epsilon =\left( \frac{ d \log T  }{2T}\right)^{(1-2\kappa)/2}$. Then, with the previous proposition, the simple regret is bounded by
	\begin{equation*}
	\EE[r_{T}]	\leq  \oO\left(d^{3/2 - \kappa}  \left({\log T}/{T}  \right)^{1/2 - \kappa} \right).
	\end{equation*}

Importantly, the bound has only a \emph{polynomial} dependence on $d$, for instance with the squared exponential kernel ($\kappa = 0$) we get $r_T \leq \oO(d^{3/2}\sqrt{\log T/T})$.

\subsection{Convergence under safety constraints}
The ability to use an arbitrary line solver for the subproblems allows us to implement safety by using a safe BO algorithm such as \textsc{SafeOpt} as a sub-solver.
We call \textsc{LineBO} with \textsc{SafeOpt} as sub-solver \textsc{SafeLineBO}. Formally, we define the safe set $\sS = \{x \in \xX | g(x) \leq  0\}$. It is unavoidable that an initial safe point $x_0 \in \sS$ must be provided. The best one can hope for is the exploration of the reachable safe set $\sS_0$, which can be defined as the connected component of $\sS$ that contains $x_0$. For details, we refer to \citet{Sui2015} and \citet{Berkenkamp16} for multiple constraints.
 
The one dimensional subproblems are guaranteed to be solved safely by the guarantees of \textsc{SafeOpt} under the same additional technical assumptions as for the original algorithm. However a natural question arises as to what extend the safe set is explored sufficiently when restricting the acquisition to one-dimensional subspaces. To allow for the possibility that a safe maximizer can be reached within one iteration from a given safe starting point, the straight line segment from this point to the optimum needs to be contained in $\sS$. Naturally, this is guaranteed if the safe set $\sS$ is convex; but other conditions are possible. For instance, if the level set $\xX_1 = \{x : f(x) > f(x_0)\} \subset \sS$ is both safe and convex, one can expect that the iterates do not leave $\xX_1$ and consequently the optimum is found. Note that this is a natural condition that arises if the function is convex on a subset of domain, as we assume for our local convergence guarantees. On the other hand, it is easy to construct counterexamples even in two dimensions that are successfully solved by \textsc{SafeOpt} but not with the \textsc{LineBO} method (for instance with a U-shaped safe set). In practice, however, this might not be a severe limitation, in particular if constraint violations are not expected close to the optimum.

\section{Practical Considerations}	
Our main goal is to provide a practical Bayesian optimization algorithm, with the main benefit that the acquisition step can be solved efficiently. We note that this enables the use of acquisition functions such as Thompson sampling or Max Value Entropy Search that rely on sampling the GP posterior and where an analytical expression is not available. Besides this, our methods has several further practical advantages, as we explain below.

\paragraph{Direction Oracles} Picking random directions is one possibility to define the sub-problems, that allows us to simultaneously obtain global and local guarantees. In practice, random directions can increase variance and by instead choosing an (approximate) descent direction it is possible to trade-off global for local exploration. An alternative way is to choose the directions coordinate aligned (\textsc{CoordinateLineBO}). This we found to be efficient on many benchmark problems, likely because of reduced variance and symmetries in the objective. If one seeks to speed up local convergence, using a gradient estimate is the obvious choice. As the gradient-norm becomes smaller, one can eventually switch to random directions to encourage random exploration. For estimating descent directions, we implement the following heuristic based on Thompson Sampling. First, we take the gradient $\tilde{g}$ at $\hat{x}_i$ of a sample from the posterior GP. Then we evaluate $\hat{x}_i + \alpha\tilde{g}$, where $\alpha$ is a small step size, and update the model. After several such steps ($\sim d$ times), we use the gradient of the posterior mean at $\hat{x}_i$ as direction oracle (see Appendix \ref{app:descent_oracle} for details). In our experiments, we found that this method (\textsc{DescentLineBO}) improves local convergence, and this variant was used on the free electron laser as well.
\paragraph{Global Model}
We introduced the \textsc{LineBO} method with a global GP model as usually done for Bayesian optimization. This has the advantage that data is \emph{shared} between the sub-problems, which can speed up convergence, but comes at the cost of inverting the kernel matrix. The iterative update cost is quadratic in the number of data points, which becomes a limiting factor typically around a few thousand steps. It is also possible to use independent sub-solvers or keep a fixed-sized data buffer; as long as the sub-problems are solved sufficiently accurately, this does not affect our theoretical guarantees and yields a further speedup. 

\paragraph{User Feedback}
An additional benefit of restricting the acquisition function to a one-dimensional subspace is that we can plot evaluations together with the model predictions on this subspace. One example that we obtained when we tuned the SwissFEL is shown in Figure \ref{fig:user_feedback}. This allows to better understand the structure of the optimization problem; moreover, it provides valuable user-feedback, as it allows to monitor model fit and to adjust GP-hyperparameters. With safety constraints this is of particular importance, as a misspecified GP model cannot capture the safe set correctly and might cause  constraint violations.

\section{Empirical Evaluation}
\subsection{Synthetic Benchmarks}

As for standard benchmarks we use the Camelback (2d) and the Hartmann6 (6d) functions. Further, we use the Gaussian $f(x) = -\exp(-4\|x\|_2^2)$ in 10 dimensions as a benchmark where local convergence is sufficient; note that when restricted to a small enough Euclidean ball this function is strongly convex. To obtain benchmarks with invariant subspaces, we augment the Camelback and Hartmann6 function with 10 and 14 auxiliary dimensions respectively, and shuffle the coordinates randomly. For the constraint case, we add an upper bound to the objective, ie $g(x) = -f(x) + \tau$ for some threshold $\tau$. We found that the performance of the local methods (including our approach) depends on the initial point. For that reason, we randomized the initial points for the Camelback and Hartmann6 function uniformly in the domain; in the constrained case restricted to the safe-set. On the Gaussian function we randomize on the level set $\{x : f(x) = y_0 \}$ with $y_0 = -0.2$ in the unconstrained and $y_0=-0.4$ in the constrained case. On all experiments we add Gaussian noise with standard deviation 0.2, to obtain a similar signal-noise ratio as on our real-world application. 

We compare our approach to random search, Nelder-Mead, SPSA, CMAES and standard GP-UCB.  For the subspace problems we additionally compare to REMBO and its interleaved variant. The latter never perform better in our experiments and is omitted from the plots for visual clarity. In the constrained case we compare to \textsc{SafeOpt} in 2 dimensions, and to the \textsc{SwarmSafeOpt} heuristic on the higher-dimensional benchmarks. We use public libraries where available, details can be found in Appendix \ref{app:synthetic_experiments}. For our \textsc{LineBO} methods, we use the UCB acquisition function. We manually chose reasonable values for hyperparameters of the methods or use recommended setting where available, but we did not run an exhaustive hyperparameter search (which would arguable not be possible in most real-world applications). All methods that use GPs share the same hyperparameters, expect on the Gaussian, where a smaller lengthscale for GP-UCB resulted in better performance.

We evaluate progress using simple regret. All regret plots show a fair comparison in terms of the total number of function evaluations on the x-axis. To compute the simple regret, each method suggests a candidate solution in each iteration (in addition to the optimization step), which is evaluated but not used in the optimization. Naturally for the GP-methods, this was chosen as the best mean of the model, and for our line methods, the best mean was determined on the current subspace. The Nelder-Mead and SPSA implementation we used did not have such an option, so progress on each evaluation is shown. Each experiment was repeated 100 times and confidence bars show the standard error.

\begin{figure*}[t]
	\centering
	\begin{subfigure}[t]{0.33\textwidth}
		\includegraphics{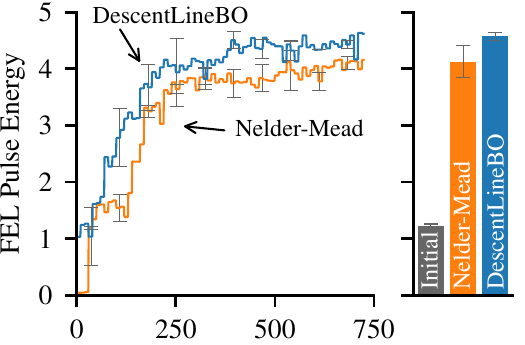}
		\caption{24 parameters}
		\label{fig:swissfel24d}
	\end{subfigure}
	\begin{subfigure}[t]{0.33\textwidth}
		\includegraphics{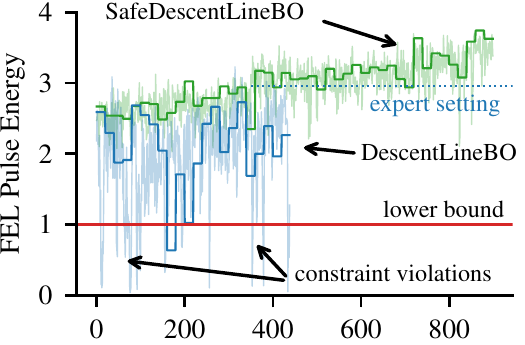}
		\caption{24 parameters with constraint}
		\label{fig:swissfel24d_safe}
	\end{subfigure}
	\begin{subfigure}[t]{0.33\textwidth}
		\includegraphics[width=160px, trim={40px 10px 40px 95px},clip]{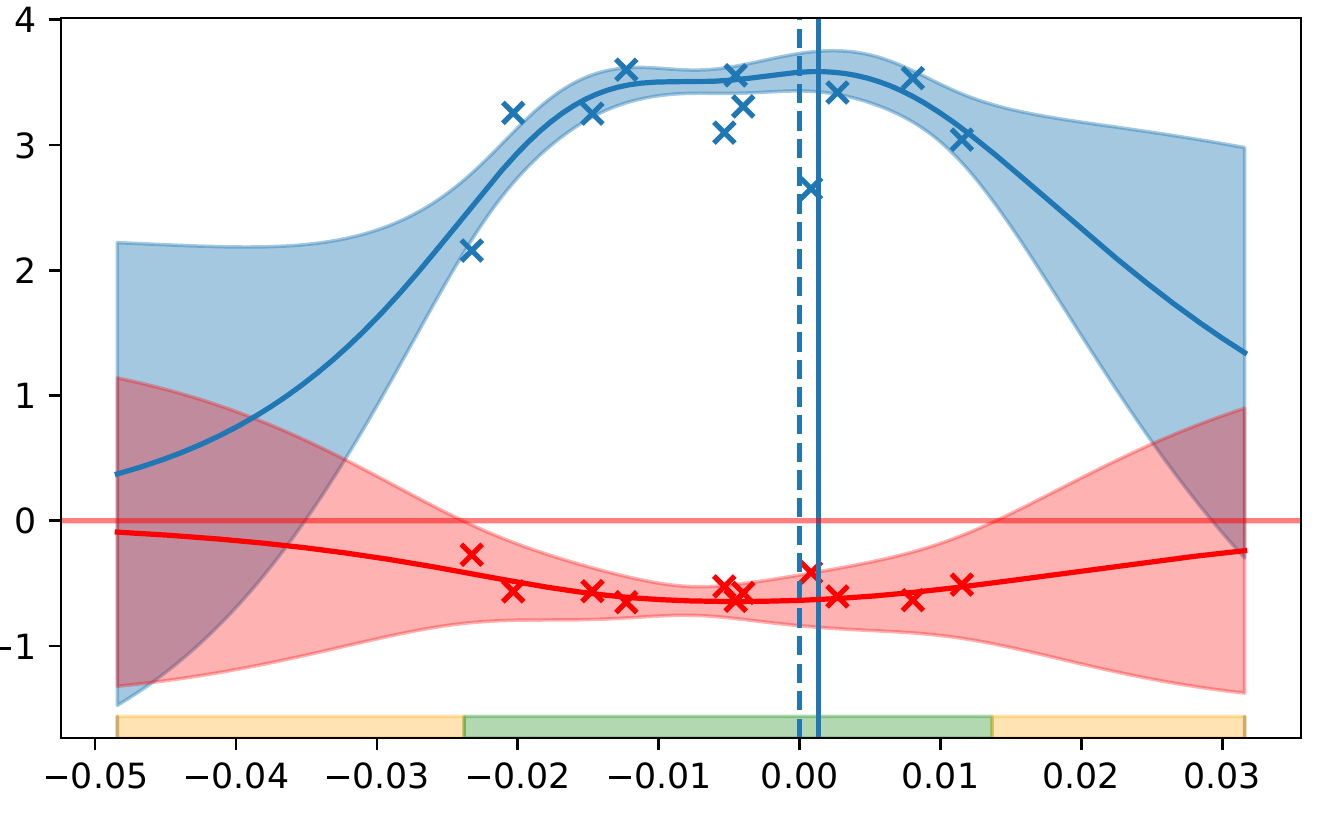}
		\caption{User feedback}
		\label{fig:user_feedback}
	\end{subfigure}
	\caption{Experiments on the Swiss Free Electron Laser (SwissFEL). (\subref{fig:swissfel24d}) Comparison of Nelder-Mead and \textsc{DescentLineBO}. (\subref{fig:swissfel24d_safe}) Optimization with safety constraints (here \textsc{DescentLineBO} was stopped early). (\subref{fig:user_feedback}) Slice plot provided as user feedback; these allow to monitor the GP fit and adjust hyper-parameters (red: safety constraints, blue: objective, crosses: line evaluations). Note that the model predictions are a slice of the global model, which also depends on observations from previous lines.}
	\label{fig:swissfel_results}
\end{figure*}

The results for the unconstrained case are presented in Figure \ref{fig:results_unconstraint}. In the standard Camelback and Hartmann6 benchmarks, we obtain competitive performance. In particular the \textsc{CoordinateLineBO} method works well, which might be due to symmetries in the benchmarks. The benchmark on the Gaussian function is challenging in that the initial signal is of the same magnitude as the noise. If an optimization algorithm initially takes steps away from the optimum, the objective quickly gets very flat, and it becomes difficult to recover by means of a gradient signal only. We found that our method allow to robustly take steps towards the optimum, where local-convergence can be guaranteed, outperforming the standard GP-UCB and CMAES algorithm. Note that the \textsc{DescentLineBO} method works particularly well on this example, as it is designed to use the estimated gradient as line directions; but it does not necessarily perform better on the other benchmarks. When adding an invariant subspace (Figure \ref{fig:results_unconstraint} \subref{fig:camelback_sub}, \subref{fig:hartmann6_sub}), our methods remain competitive with the bulk of methods, but surprisingly also UCB works very well on the camelback function with augmented coordinates. This might be due to an effect similar as in Proposition \ref{prop:global} carrying over from the random restarts of the approximate acquisition function optimizer. 

Figure \ref{fig:evaltime} shows computation time per iteration in a 10 dimensional setting (Hartmann6d+4d) averaged over 500 steps. Our methods obtain roughly one order of magnitude speed up compared to the full-scale Bayesian optimization methods; however this is of course dependent on the implementation. For GP-UCB and REMBO, we optimize the acquisition function using L-BFGS with 50 restarts, where starting points are either randomly chosen or from a previous maximizer.

The results for the constrained case can be found in Figure \ref{fig:constraint}. Our methods clearly outperform both \textsc{SafeOpt} and \textsc{SwarmSafeOpt} in terms of simple regret.

\subsection{Tuning the Swiss Free Electron Laser}
Parameter tuning is a tedious and repetitive task for operation of free electron lasers. The main objective is to increase the laser energy measured by a gas detector at the end of the beam-line. Among hundreds of available parameters that expert operators usually adjust, some parameter groups allow for automated tuning. Those include quadrupole currents settings, beam position parameters and configuration variables of the undulators. For our tests, a suitable subset of 5-40 parameters was selected by machine experts. The machine is operated at 25 Hz and we averaged 10 consecutive evaluations to reduce noise. Ideally, the computation time per step is well below 1s to avoid slowing down the overall optimization. This effectively rules out full-scale Bayesian optimization given the number of parameters. Besides manual tuning by operators, a random walk optimizer is in use and reported to often achieve satisfactory performance when run over a longer period of time; in other cases it did not improve the signal while other methods did. This hints that hill-climbing on the objective should be taken into account as a feasible step towards an acceptable solution, but global exploration and noise robustness are important, too. Nelder-Mead is mostly considered as standard benchmark in the accelerator community. Standard Bayesian optimization was previously reported to outperform it \citep{McIntire2016}, but safety constraints, and the efficient scaling to high dimensions were not considered.
Safe operation constraints include electron loss monitors and a lower threshold on the pulse energy, which is important to maintain during user operation. For our experiments we were mainly concerned with the latter, as at the time of testing, the loss monitoring system could not be used for technical reasons; but this will be an important addition once implemented.

Our results are shown in Figure \ref{fig:swissfel24d}. To obtain a systematic comparison, we manually detuned the machine, then run both Nelder-Mead and \textsc{DescentLineBO} twice from the same starting point (limited machine development time did not allow for a more extensive comparison). Our method soundly outperforms Nelder-Mead, both in terms of convergence speed and pulse energy at the final solution. A direct comparison between the \textsc{LineBO} and \textsc{SafeLineBO} in Figure \ref{fig:swissfel24d_safe} shows that the safe method is able to maintain the safety constraint. The safety constraint has the additional benefit of restricting the search space which we found to improve convergence in this case. The solution obtained after 600 steps (after $\sim15$ min) already achieves a higher pulse energy than the previous expert setting, which was obtained with the help of a local random walk optimizer. A single, successful run with 40 parameters can be found in Figure \ref{fig:swissfel40d}.

\section{Conclusion}
We presented a novel and practical Bayesian optimization algorithm, \textsc{LineBO}, which iteratively decomposes the problem to a sequence of one dimensional sub-problems. This addresses the often ignored issue of how to maximize the acquisition function, and allows to scale the method to high-dimensional settings. We showed that the algorithm is theoretically as well as practically effective. In addition, it can also be used with safety constraints by means of safely solving each sub-problem, and is therefore, to the best of our knowledge, the first method to achieve this. Finally, we demonstrated how we apply the \textsc{SafeLineBO} method on SwissFEL for tuning the pulse energy with up to 40 parameters on a continuous domain while satisfying safe operation constraints.

\section*{Acknowledgements}
The authors thank in particular Manuel Nonnenmacher for the work he did during his Master thesis and Kfir Levy for valuable discussions and feedback. For the experiments on the free electron laser, the authors would like to acknowledge the support of the entire SwissFEL team.

This research was supported by SNSF grant 200020\_159557 and 407540\_167212 through the NRP 75 Big Data program. Further, this project has received funding from the European Research Council (ERC) under the European Union’s Horizon 2020 research and innovation programme grant agreement No 815943. 

%
\bibliography{references.bib}
\bibliographystyle{icml2019_style/icml2019}

\newpage
\newcommand{\Vol}{\text{Vol}}
\appendix
\section{Bayesian Optimization}\label{app:BO}

A general outline of Bayesian optimization is given in Algorithm \ref{alg:BO} below. The evaluation point is determined using an acquisition function $\alpha(x|\mM)$ that typically depends on the GP model $\mM$. One of the most common acquisition functions, that has been analyzed theoretically, is GP-UCB \cite{Srinivas2009}. With posterior mean $\hat{f}_t$ and posterior standard deviation $\sigma_t(x)$ it is defined as the lower confidence bound (in the minimization setting),
\begin{align}\label{eq:ucb}
\alpha(x) = \hat{f}_t(x) - \beta_t \sigma_t(x)
\end{align}
for a scaling factor $\beta_t$ (for details, see \citet{Srinivas2009}).

\begin{algorithm}[h]
	\caption{Bayesian Optimization (BO)}\label{alg:BO}
	\begin{algorithmic}[1]
		\REQUIRE Domain $\xX$, accuracy $\epsilon$, acquisition function $\alpha$
		\\GP prior $\mM_0 = GP(\mu_0, k_0)$
		\FOR{$t=1,2,3,\dots$}
		\STATE $x_t \gets \argmin_{x \in \xX} \alpha(x|\mM_{t-1})$\hfill \textit{// acquisition step}
		\STATE $y_t \gets f(x_t) + \epsilon$ \hfill \textit{// obtain observation}
		\STATE $\mM_t \gets \mM_{t-1}|(x_t,y_t)$ \hfill \textit{// update posterior}
		\STATE $\hat{x}_t = \argmin_{x \in \xX} err(x)$ \hfill \textit{// best point, eq. \ref{eq:gp_err} }
		\IF{$err(\hat{x}_t) \leq \epsilon$}
		\RETURN $\hat{x}_t$, $\mM_t$ \hfill \textit{// best point,  posterior model}
		\ENDIF
		\ENDFOR
	\end{algorithmic} 
\end{algorithm}
\subsection{Sample Complexity of Bayesian Optimization}\label{app:sample_complexity_bo}
In the frequentist analysis with confidence intervals $\hat{f}(x) \pm \sigma(x)$, the simple regret at any point $x\in \xX$ can be controlled with
\begin{align}
err(x) := \hat{f}(x) + \sigma(x) - (\min_{x'\in \xX} \hat{f}(x') - \sigma(x')) \text{ .} \label{eq:gp_err}
\end{align}
The sample complexity bounds of GP-UCB make sure that the breaking condition in Algorithm \ref{alg:BO} is eventually satisfied. Even though these bounds are typically formulated for cumulative regret, the proofs in fact bound the following quantity,
\begin{align*}
\sum_{t=1}^T err(x_t) \leq \oO(\gamma_T \sqrt{T}) \text{ ,}
\end{align*}
see \citep[Theorem 3]{Chowdhury2017}. From this a bound on simple regret follows as
\begin{align*}
r_T \leq err(\hat{x}_T) \leq \frac{1}{T} \sum_{t=1}^T err(x_t) \leq \frac{\gamma_T}{\sqrt{T}} \text{ .}
\end{align*}

Bounds on $\gamma_T$ are known for different kernels, including for the linear kernel: $\gamma_T \leq \oO(d\sqrt{T})$, the RBF kernel: $\gamma_T \leq \oO((\log(T))^{d+1})$, and the Matern kernel with $\nu > 1$: $\gamma_T \leq \oO\big(T^{d(d+1)/(2\nu + d(d+1))}(\log(T))\big)$; see \citet[Theorem 5]{Srinivas2009}.

\subsection{Safe Bayesian Optimization}\label{app:safeopt}

\textsc{SafeOpt} uses an idea that is similar to Algorithm \ref{alg:BO}. A GP-model is used to estimate the implicit constraint function $g$ with confidence intervals $\hat{g}(x) \pm \beta_t\sigma^{g}_t(x)$. The confidence estimates can be used to define a conservatively estimated safe set $\hat{\sS} = \{  x \in \xX : \hat{g}_t(x) + \beta_t\sigma^{g}_t(x) \leq 0 \}$. The acquisition step is then restricted to $\hat{\sS}$, and under the condition that the confidence estimates hold, the algorithm does not violate the constraints. However, the exploration problem becomes more difficult, as both $f$ and $g$ need to be explored in an appropriate way.  For completeness, we reproduce the pseudo-code from \cite{Sui2015,Berkenkamp16} in Algorithm \ref{alg:SafeOpt}. Please refer to the original publication for a more detailed treatment.

In each iteration $t$, the algorithm defines a safe set $\sS_t$, the set of potential minimizers $M_t$, and the expander set $G_t$ with points that can possibly enlarge the safe set. The algorithm uses two functions; the first is the uncertainty at a specific point $x \in \xX$ to determine which points to acquire,
\begin{equation}
	w_t(x) = \max(2\beta_t\sigma^{f}_t(x),2\beta_t\sigma^{g}_t(x)) \text{ .}
\end{equation}
Next, to quantify possible expanders of the safe set,
\begin{equation}\label{eq:p_t}
	p_t(x) = |\{ x \in \xX \setminus \sS_t | \hat{g}_t(x) + \beta\sigma^{g}_t(x) - L\norm{x}_2\geq h  \}| \text{ .}
\end{equation}
Also, denote $u_t(x) = \hat{f}_t(x) +\beta_t\sigma_t^f(x)$ and $l_t(x) = \hat{f}_t(x) - \beta_t\sigma_t^f(x)$ the lower and upper confidence bound of $f$.

\begin{algorithm}[h]
	\caption{\textsc{SafeOpt}}\label{alg:SafeOpt}
	\begin{algorithmic}[1]
		\REQUIRE Domain $\xX$, initial safe set $\sS_0$, safety threshold $h$, Lipschitz constant $L$, GP priors $\mM_0$ for both $f$ and $g$, 
		\FOR{$t=1,2,3,\dots$}
			\STATE $\sS_t \gets \cup_{x\in S_{t-1}} \{ x' \in \xX |\hat{g}_t(x) - L\norm{x-x'}_2 \geq h \}$
			\hfill \\ \textit{ // Enlarge safe set }
			\STATE $G_t \gets \{ x\in \sS_t  | p_t(x) > 0 \text{ as in } \eqref{eq:p_t} \}$ 	\hfill \textit{ // expander set}
			\STATE $M_t \gets \{ x\in \sS_t | l_t(x) \leq \min_{x'\in S_t} u_t(x') \}$
							\\ 	\hfill  \textit{ // plausible minimizer set }
			\STATE $x_t \gets \arg\max_{x \in G_t\cup M_t} w_t(x)$
									\\ 	\hfill  \textit{ // uncertainty sampling }
			\STATE $y_t \gets f(x_t) + \epsilon$ \hfill  \textit{ // acquire observations }
			\STATE $s_t \gets g(x_t) + \epsilon$
			\STATE $\mM_t \gets \mM_{t-1}|(x_t,y_t, s_t)$
				\hfill  \textit{// update posterior }
			\STATE $\hat{x}_t = \argmin_{x \in \xX} err(x)$ \hfill \textit{ // best point, eq. \ref{eq:gp_err} }
			\IF{$err(\hat{x}_t) \leq \epsilon$}
			\RETURN $\hat{x}_t$, $\mM_t$ \hfill \textit{// best point,  posterior model}
			\ENDIF
		\ENDFOR
	\end{algorithmic} 
\end{algorithm}

The algorithm in its original formulation requires the knowledge of a Lipschitz constant $L$. It is however possible to derive a bound on $L$ from the norm bound $\|f\|_\hH$ as by Assumption \ref{ass:f_regular}. Note that this algorithm is in particularly simple to implement in the one-dimensional setting. There, the safe-set is always an interval with its endpoints being possible expanders.

\section{Proofs of Theoretical Results}

\subsection{Global Convergence}\label{app:global_convergence}

We first show the following lemma.
\begin{lemma}\label{lemma:global_lemma}
	Let $f \in \hH_k$ be twice differentiable with effective dimension $d_e$. Further, let $f_K^* = \min_{x \in \lL_1, \dots, \lL_K} f(x)$ be the minimum objective value that can be obtained by minimizing any line up to iteration $K$. Then,
	\begin{equation}
	\PP[f_K^* - f^* \leq \tau] \geq 1 - \exp(-K\xi(\tau))
	\end{equation}
	where $\xi(\tau)$ is a lower bound on the probability that a random line intersects the set $V_\tau = \{ x \in \xX : f(x) \leq f^* + \tau \}$. Further, if the first order condition at the minimum $x^*$ is met, then  $\xi(\tau) = \Omega\left(\tau^{\frac{d_e-1}{2}}\right)$
\end{lemma}

Before we go on to the proof, we show how Proposition \ref{prop:global} follows.
\begin{proof}[Proof of Proposition \ref{prop:global}]
 	Denote $f_K^* = \min_{x \in \lL_1, \dots, \lL_K} f(x)$ and remember that each line is solved up to $\epsilon$ accuracy, therefore $f(\hat{x}_K)  \leq f_K^* + \epsilon$. Using Lemma \ref{lemma:global_lemma}, we know that with probability at least $1 - \exp(-K\xi(\tau))$, $f(\hat{x}) - f^* \leq \epsilon + \tau$, hence when $\exp(-K\xi(\tau)) = \delta$ the statement in the proposition is true. Solving for $\tau$ yields $\tau \leq \oO\left( \frac{1}{K} \log\left(\frac{1}{\delta} \right)\right)^{2/(d_e-1)}$, concluding the proof.
\end{proof}

\begin{proof}[Proof of Lemma \ref{lemma:global_lemma}]
	Let $\xi(\tau)$ be a lower bound on the probability that a random line intersects the set $V_\tau = \{ x \in \xX : f(x) \leq f^* + \tau \}$. Using this, we find
	\begin{eqnarray*}
		&& \PP[f_K^* - f^* \geq \tau]\\
		& = & \PP[f_1^* - f^* \geq \tau \wedge \dots \wedge f_K^* - f^* \geq \tau] \\	
		&=&\prod_{i=1}^K \PP[f_i^* - f^* \geq \tau | {x_{i-1},y_{i-1},\dots, x_1,y_1}]\\
		&\leq& (1-\xi(\tau))^K \leq \exp(-K\xi(\tau))
	\end{eqnarray*}
where the last inequality uses $1-x \leq e^{-x}$. Hence
	\begin{align*}
		\PP[f_K^* - f^* \leq \tau] \geq 1 - (1-\xi(\tau))^K \geq 1 - e^{-\xi(\tau) K} \text{ .}
	\end{align*}

	Using the assumption that $f$ is twice-differentiable, we can over-approximate $f$ around a minimizer $x^*$ using a quadratic function
	$f(x^*+h) \leq f(x^*) + \frac{\alpha}{2}\norm{h}^2$ for small $h$ and $\alpha > 0$. Therefore $\tilde{V}_\tau := \{ x \in \xX |\frac{\alpha}{2}\norm{x - x^*}^2 \leq \tau\}) \subset V_\tau$, and it is enough to intersect $\tilde{V}_\tau$ with a random line. Note that if we also use the assumption that the function varies only in $d_e$ dimensions, we can restrict the approximation to the active subspace, therefore $\Vol(\tilde{V}_\tau) \geq \Omega(\tau^{d_e/2})$. If we allow the hidden constant also to depend on the diameter of the domain $\xX$, the probability that a random line intersects $\tilde{V}_\tau$, and therefore $V_\tau$, is at least  $\Omega(\tau^{(d_e-1)/2})$.
\end{proof}
\begin{remark}
	The assumption that $f$ is twice differentiable is always satisfied if the kernel function $k$ is twice differentiable, see \citet[Corollary 4.36]{Steinwart2008}.	
\end{remark}

\subsection{Local Convergence}\label{app:local_convergence}
First, we recall the definition of smooth and strongly convex functions.
\begin{definition}[Strong Convexity]\label{def:strongconvex}
 A differentiable function $f:\xX \rightarrow \RR$ is called $\alpha$-strongly convex if there exists $\alpha > 0$ such that for all$ x,h \in \xX \subseteq \mathbb{R}^d$,
	\begin{eqnarray}
	\braket{\nabla f(x),h}+\frac{\alpha}{2}\norm{h}^2_2   & \leq & f(x+h) -f(x) \label{eq:strgcnvx}
	\end{eqnarray}

\end{definition}
\begin{definition}[Smoothness]\label{def:smooth} A differentiable function $f:\xX \rightarrow \RR$ is called $\beta$-smooth, if there exists $\beta > 0$ such that for all $x,h \in \xX \subseteq \mathbb{R}^d$,
	\begin{eqnarray}
f(x+h) -f(x)   & \leq  &  \braket{\nabla f(x),h}+\frac{\beta}{2}\norm{h}^2_2.  \label{eq:smooth}
	\end{eqnarray}
\end{definition}

Strong convexity implies the Polyak–Lojasiewicz condition,
\begin{equation}\label{eq:strgcnvx2}
f(x)-f(x^*)\leq\frac{1}{2\alpha}\braket{\nabla f(x),\nabla f(x)} \text{ .}
\end{equation}

The next lemma shows that randomly chosen directions can be used as descent directions (Lemma \ref{lem:random_directions} in the main text).

\begin{lemma}[Random Descent Direction ] \label{lem:random_line}
	Let $l \in \RR^d$ be a randomly chosen direction. Specifically assume that $l$ is uniformly random on the $d$-dimensional unit sphere (random directions) or uniformly among an orthonormal basis (coordinate descent). Then for any $g \in \RR^d$
	
	\begin{equation*}\label{eq:random_line}
	\EE[\< g, l \>^2] = \frac{1}{d} \|g\|^2 \text{ .}
	\end{equation*}
\end{lemma}
\begin{proof}[Proof of Lemma \ref{lem:random_line}]
	Denote by $l_i$ the $i$th coordinate of $l$. Note that $\EE[\sum l_i^2] = 1$, hence $\EE[l_i^2] = \frac{1}{d}$. Further $\EE[l_il_j] = \EE[\EE[l_i|l_j]] = 0$ for $i\neq j$ due to symmetry argument if $l$ is uniformly on the sphere, and by orthonormality in the coordinate case. The result follows from expanding the square and using the previous two equations.	
\end{proof}

\begin{lemma}[Exact line search oracle] \label{lem:exact_oracle}
	Let $f$ be $\alpha$-strongly convex and $\beta$-smooth on a domain $\xX$. If we obtain iterates $\hat{x}_i$ from Algorithm \ref{alg:linebo} with random directions $l_i$ that satisfy Lemma \ref{lem:random_line}, then the exact line-search solution $x_{i+1}^* = \argmin_{x \in \lL_{i}} f(x)$ improves per step by
	\begin{equation*}
	\EE[f(x_{i+1}^*) - f(x^*)] \leq  \left(1-\frac{\alpha}{\beta d}\right)(f(\hat{x}_i) - f(x^*)) \text{ .}
	\end{equation*}
\end{lemma}

\begin{proof}[Proof of Lemma \ref{lem:exact_oracle}]
	Let $x_{i+1}^* = \argmin_{x \in \lL_{i}} f(x)$ be the solution obtained from an exact line-search on the sub-problem $\lL_{i}$. This implies that for any $h \in \RR$,
\begin{align*} 
f(x^*_{i+1}) - f(\hat{x}_i) \leq f(\hat{x}_i + hl_i) - f(\hat{x}_i)
\end{align*}
Further assume that the directions $l_i$ are random, satisfy Lemma \ref{lem:random_line} and $\|l_i\| = 1$. Smoothness implies that
\begin{align*}
f(\hat{x}_i + hl_i) - f(\hat{x}_i)  \stackrel{\text{Def }(3)} \leq \braket{\nabla f(\hat{x}_i), h l_i} + \frac{\beta}{2} h^2 \norm{l_i}^2_2
\end{align*}	
In particular, the inequality also holds for $h = -\frac{\braket{\nabla f(\hat{x}_i),l_i}}{\beta}$, and note that $l_i \in \RR^d$ is normalized, ie $\|l_i\|=1$, hence taking the previous two inequalities together,
\begin{align*}
f(x^*_{i+1}) - f(\hat{x}_i)	 \leq  -\frac{\braket{\nabla f(\hat{x}_i),l_i}^2}{2\beta} 
\end{align*}	
Taking expectation over the random direction $l_i \in \RR^d$ and using Lemma \ref{lem:random_line}, we get
\begin{align*}
	\EE [f(x^*_{i+1})] - f(\hat{x}_i)  & \leq  -\frac{1}{2\beta d} \norm{\nabla f(x_i)}^2_2\\
	&\leq - \frac{\alpha}{\beta d} (f(\hat{x}_i) - f(x^*))  \text{ ,}
\end{align*}
and the last inequality uses the Polyak–Lojasiewicz condition \eqref{eq:strgcnvx2}. Rearranging concludes the proof, 
\begin{eqnarray*}
	\EE[f(x_{i+1}^*) - f(x^*)] \leq  \left(1-\frac{\alpha}{\beta d}\right)(f(\hat{x}_i) - f(x^*)) \text{ .}
\end{eqnarray*}
\end{proof}

	\begin{proof}[Proof of Proposition \ref{prop:local_convg}]
	Assume that we run Algorithm \ref{alg:BO} with accuracy $\epsilon$ and obtain iterates $\hat{x}_i$ that do not leave the subset $\xX_c$ where the function is $\beta$-smooth and $\alpha$-strongly convex. Denote the exact line-search solutions by $x_{i+1}^* = \argmin_{x \in \lL_{i+1}} f(x)$	and $\gamma = \frac{\alpha}{\beta d}$, to find
	\begin{eqnarray*}
		\EE[f(\hat{x}_{i+1})] -f(x^*) & \leq & \EE[f(x^*_{i+1}) - f(x^*)] + \epsilon \\
		& \leq & (1-\gamma)(f(\hat{x}_i) - f(x^*)) + \epsilon  \text{ ,}
	\end{eqnarray*}
	by means of Lemma \ref{lem:exact_oracle}. Recursively applying the previous inequality gives
\begin{eqnarray*}
		\EE[f(x_{K})] -f(x^*) & \leq & \epsilon \sum_{i=0}^{K-1} (1-\gamma)^i \\ & & \quad+  (1-\gamma)^K(f(\hat{x}_0) - f(x^*)) \\
		& \leq &  \frac{\epsilon }{\gamma} + (1-\gamma)^K(f(\hat{x}_0) - f(x^*)) \\
		& \leq &  \frac{\epsilon }{\gamma} + \exp(-K\gamma)(f(\texttt{}x_0) - f(x^*)) 
	\end{eqnarray*}
This concludes the proof.
\end{proof}

\section{Synthetic Experiments}\label{app:synthetic_experiments}
\subsection{Implementation Details}
For the \textsc{Random} method, we pick points uniformly in the domain and report the best observation as candidates, without any control on the noise. UCB is implemented using GPy \cite{gpy2014}, and the acquisition function is maximized using the L-BFGS solver provided by the SciPy library. To evade local maxima, 50 restarts are used, both containing random points and previous maximizers. For Nelder-Mead we use the SciPy implementation. SPSA is provided by noisyopt 
(\href{https://noisyopt.readthedocs.io}{https://noisyopt.readthedocs.io}). CMAES is provided by the pycma package (\href{https://github.com/CMA-ES/pycma}{https://github.com/CMA-ES/pycma}). Our  REMBO and InterleavedREMBO implementation is based on  \href{https://github.com/jmetzen/bayesian\_optimization}{https://github.com/jmetzen/bayesian\_optimization}. SafeOpt and SwarmSafeOpt use the author implementation \href{https://github.com/befelix/SafeOpt}{https://github.com/befelix/SafeOpt}.

\subsection{Direction Oracle} \label{app:descent_oracle}
Our \textsc{DescentLineBO} algorithm uses the following heuristic to find the directions. We use a step-size of $\alpha=0.1$ and $m=2d$ evaluations in our experiments. Note that this can be seen as Thompson sampling on the Euclidean ball $B_\alpha(\hat{x})$ with a linear approximation of the posterior GP.
\begin{algorithm}[h]
	\caption{Descent Direction Oracle}\label{alg:descentoracle}
	\begin{algorithmic}[1]
		\REQUIRE Current point $\hat{x}$, step size $\alpha$, number of evaluations $m$, GP model $GP(\mu_0, k_0)$.
		\FOR{$i=1,2,\dots,m$}
		\STATE $\tilde{f} \sim GP(\mu_{i-1}, k_{i-1})$
		\STATE $\tilde{g} \gets \nabla \tilde{f}(\hat{x})$  \hfill \textit{// sample gradient at $x$}
		\STATE $x_i \gets \hat{x} - \alpha\tilde{g}$ \hfill \textit{// linear approximation on $B_\alpha(\hat{x})$}
		\STATE $y_i \gets f(x_t) + \epsilon$  \hfill \textit{// obtain observation}
		\STATE $\mu_i,k_i \gets (\mu_{i-1}, k_{i-1})|(x_t, y_t)$ \hfill \textit{// update posterior}
		\ENDFOR
		\RETURN $\nabla \mu_m(\hat{x})$ \hfill \textit{// gradient of posterior mean}
	\end{algorithmic} 
\end{algorithm}

\end{document}